\documentclass{scrartcl}

\usepackage[utf8]{inputenc}
\usepackage[english]{babel}

\usepackage[
colorlinks = True,
linkcolor = red,
citecolor = blue]{hyperref}

\usepackage{amsmath, amsfonts, amsthm}
\usepackage[linesnumbered, ruled, boxed]{algorithm2e}
\usepackage{graphicx}
\usepackage{tikz}
\usepackage{booktabs}
\usepackage{indentfirst}
\usepackage{natbib}
\usepackage[title]{appendix}

\newtheoremstyle{seriftheorem}
  {1.25\topsep}   				
  {1.25\topsep}   				
  {\normalfont}  				
  {}       						
  {\sffamily \bfseries} 		
  {.}         					
  {.5em} 						
  {}          					

\theoremstyle{seriftheorem}
\newtheorem{myprp}{Proposition}
\newtheorem{mylem}{Lemma}

\begin{document}

\title{A Comparative Study of Gamma Markov Chains for Temporal Non-Negative Factorization}

\author{
Louis Filstroff\,$^{1}$ \hspace{0.5em} Olivier Gouvert\,$^2$ \hspace{0.5em} Cédric Févotte\,$^3$ \hspace{0.5em} Olivier Cappé\,$^4$ \\
\small $^1$ Department of Computer Science, School of Science, Aalto University, Finland \\
\small $^2$ Mila - Quebec Artificial Intelligence Institute \\
\small $^3$ IRIT, Université de Toulouse, CNRS, France \\
\small $^4$ DI ENS, CNRS, INRIA, Université PSL
}

\maketitle

\begin{abstract}
Non-negative matrix factorization (NMF) has become a well-established class of methods for the analysis of non-negative data. In particular, a lot of effort has been devoted to probabilistic NMF, namely estimation or inference tasks in probabilistic models describing the data, based for example on Poisson or exponential likelihoods. When dealing with time series data, several works have proposed to model the evolution of the activation coefficients as a non-negative Markov chain, most of the time in relation with the Gamma distribution, giving rise to so-called temporal NMF models. In this paper, we review four Gamma Markov chains of the NMF literature, and show that they all share the same drawback: the absence of a well-defined stationary distribution. We then introduce a fifth process, an overlooked model of the time series literature named BGAR(1), which overcomes this limitation. These temporal NMF models are then compared in a MAP framework on a prediction task, in the context of the Poisson likelihood.
\par\vskip\baselineskip\noindent
\end{abstract}

\section{Introduction} \label{sec-intro}

\subsection{Non-negative matrix factorization}

Non-negative matrix factorization (NMF) \citep{paatero1994positive,lee1999learning} has become a widely used class of methods for analyzing non-negative data. Let us consider $N$ samples in $\mathbb{R}^{F}_{+}$. We can store these samples column-wise in a matrix, which we denote by $\mathbf{V}$ (therefore of size $F \times N$). Broadly speaking, NMF aims at finding an approximation of $\mathbf{V}$ as the product of two non-negative matrices:
\begin{equation}
\mathbf{V} \simeq \mathbf{WH},
\label{eq-nmf-approx}
\end{equation}
where $\mathbf{W}$ is of size $F \times K$, and $\mathbf{H}$ is of size $K \times N$. $\mathbf{W}$ and $\mathbf{H}$ are referred to as the dictionary and the activation matrix, respectively. The factorization rank $K$ is usually chosen such that $K \ll \min(F,N)$, hence producing a low-rank approximation of $\mathbf{V}$. This factorization is often retrieved as the solution of an optimization problem, which we can write as:
\begin{equation}
\min_{\mathbf{W} \geq 0,~\mathbf{H} \geq 0} D(\mathbf{V}|\mathbf{WH}),
\label{eq-nmf-min}
\end{equation}
where $D$ is a measure of fit between $\mathbf{V}$ and its approximation $\mathbf{WH}$, and the notation $\mathbf{A} \geq 0$ denotes the non-negativity of the entries of the matrix $\mathbf{A}$. One of the key aspects to the success of NMF is that the non-negativity of the factors $\mathbf{W}$ and $\mathbf{H}$ yields an interpretable, part-based representation of each sample: $\mathbf{v}_n \simeq \mathbf{Wh}_n$ \citep{lee1999learning}.

Various measures of fit have been considered in the literature, for instance the family of $\beta$-divergences \citep{fevotte2011algorithms}, which includes some of the most popular cost functions in NMF, such as the squared Euclidian distance, the generalized Kullback-Leibler divergence, or the Itakura-Saito divergence. As it turns out, for many of these cost functions, the optimization problem described in Eq.~\eqref{eq-nmf-min} can be shown to be equivalent to the joint maximum likelihood estimation of the factors $\mathbf{W}$ and $\mathbf{H}$ in a statistical model, that is:
\begin{equation}
\max_{\mathbf{W},\mathbf{H}} p(\mathbf{V}|\mathbf{W},\mathbf{H}).
\label{eq-nmf-max}
\end{equation}
This leads the way to so-called \textit{probabilistic} NMF, i.e., estimation or inference tasks in probabilistic models whose observation distribution may be written as:
\begin{equation}
\mathbf{v}_{n} \sim p(~.~;\mathbf{Wh}_{n}, \boldsymbol{\Theta}), \quad \mathbf{W} \geq 0, \quad \mathbf{H} \geq 0,
\label{eq-prob-nmf}
\end{equation}
that is to say that the distribution of $\mathbf{v}_n$ is parametrized by the dot product of the factors $\mathbf{W}$ and $\mathbf{h}_n$. Other potential parameters of the distribution are generically denoted by $\boldsymbol{\Theta}$. Most of the time these distributions are such that $\mathbb{E}(\mathbf{v}_n) = \mathbf{Wh}_n$.

This large family encompasses many well-known models of the literature, for example models based on the Gaussian likelihood \citep{schmidt2009bayesian} or the exponential likelihood \citep{fevotte2009nonnegative,hoffman2010bayesian}. It also includes factorization models for count data, which are most of the time based on the Poisson distribution\footnote{These models are sometimes generically referred to as ``Poisson factorization'' or ``Poisson factor analysis''.} \citep{canny2004gap,cemgil2009bayesian,zhou2012beta,gopalan2015scalable}, but can also make use of distributions with a larger tail, e.g., the negative binomial distribution \citep{zhou2018nonparametric}. Finally, more complex models using the compound Poisson distribution have been considered \citep{simsekli2013learning,basbug2016hierarchical,gouvert2019recommendation}, allowing to extend the use of the Poisson distribution to various supports $(\mathbb{N}, \mathbb{R}_+, \mathbb{R}, \dotsc)$.

In the vast majority of the aforementioned works, prior distributions are assumed on the factors $\mathbf{W}$ and $\mathbf{H}$. This is sometimes referred to as \textit{Bayesian} NMF. In this case, the columns of $\mathbf{H}$ are most of the time assumed to be independent:
\begin{equation}
p(\mathbf{H}) = \prod_{n=1}^{N} p(\mathbf{h}_n).
\label{eq-h-idp}
\end{equation}
The factors being non-negative, a standard choice is the Gamma distribution\footnote{Throughout the article, we consider the ``shape and rate'' parametrization of the Gamma distribution, i.e. $\text{Gamma}(x|\alpha, \beta) \propto x^{\alpha-1} \exp (-\beta x).$}, which can be sparsity-inducing if the shape parameter is chosen to be lower than one. The inverse Gamma distribution has also been considered.

\subsection{Temporal structure of the activation coefficients}

In this work, we are interested in the analysis of specific matrices $\mathbf{V}$ whose columns cannot be treated as exchangeable, because the samples $\mathbf{v}_n$ are correlated. Such a scenario arises in particular when the columns of $\mathbf{V}$ describe the evolution of a process over time.

From a modeling perspective, this means that correlation should be introduced in the statistical model between successive columns of $\mathbf{V}$. This can be achieved by lifting the prior independence assumption of Eq.~\eqref{eq-h-idp}, thus introducing correlation between successive columns of $\mathbf{H}$. In this paper, we consider a Markov structure on the columns of $\mathbf{H}$:
\begin{equation}
p(\mathbf{H}) = p(\mathbf{h}_1) \prod_{n \geq 2} p(\mathbf{h}_n | \mathbf{h}_{n-1}).
\label{eq-h-markov}
\end{equation}
We will refer to such a model as a $\textit{dynamical}$ NMF model. Note that recent works go beyond the Markovian assumption, i.e., assume dependency with multiple past time steps, and are labeled as ``deep'' \citep{gong2017deep,guo2018deep}.

Several works \citep{fevotte2013non,schein2016poisson,schein2019poisson} assume that the transition distribution $p(\mathbf{h}_n | \mathbf{h}_{n-1})$ makes use of a transition matrix $\boldsymbol{\Pi}$ of size $K \times K$ to capture relationships between the different components. In this case, the distribution of $h_{kn}$ depends on a linear combination of all the components at the previous time step:
\begin{equation}
p(\mathbf{h}_n | \mathbf{h}_{n-1}) = \prod_{k} p(h_{kn}|\sum_l \pi_{kl} h_{l(n-1)}).
\label{eq-trans-h-matrix}
\end{equation}

In this work, we will restrict ourselves to $\boldsymbol{\Pi} = \mathbf{I}_K$. Equivalently, this amounts to assuming that the $K$ rows of $\mathbf{H}$ are a priori independent, and we have
\begin{equation}
p(\mathbf{H}) = \prod_{k} p(h_{k1}) \prod_{n \geq 2}p(h_{kn}|h_{k(n-1)}).
\label{eq-h-markov-idp}
\end{equation}
We will refer to such a model as a \textit{temporal} NMF model.

A first way of dealing with the temporal evolution of a non-negative variable is to map it to $\mathbb{R}_+$. It is then commonly assumed that this variable evolves in Gaussian noise. This is for example exploited in the seminal work of \citet{blei2006dynamic} on the extension of latent Dirichlet allocation to allow for topic evolution\footnote{Note that this particular mapping is actually slightly more complex, as the $K$-dimensional real vector must be mapped to the $(K-1)$ simplex due to further constraints in the model.}. A similar assumption is made in \citet{charlin2015dynamic}, which introduces dynamics in the context of a Poisson likelihood (factorizing the user-item-time tensor). Gaussian assumptions allow to use well-known computational techniques, such as Kalman filtering, but result in loss of interpretability.

We will focus in this paper on naturally non-negative Markov chains. Various non-negative Markov chains have been proposed in the NMF literature (see Section~\ref{sec-study} and references therein). They are all built in relation with the Gamma (or inverse Gamma) distribution. As a matter of fact, these models exhibit the same drawback: the chains all have a degenerate stationary distribution. This can lead to undesirable behaviors, such as the instability or the degeneracy of realizations of the chains. We emphasize that this is problematic from the probabilistic perspective only, since these prior distributions may still represent an appropriate regularization in a MAP setting.

\subsection{Contributions and organization of the paper}

The contributions of this paper are 4-fold:
\begin{itemize}
	\item We review the existing non-negative Markov chains of the NMF literature and discuss some of their limitations. In particular we show that these chains all have a degenerate stationary distribution;
	\item We present an overlooked non-negative Markov chain from the time series literature, the first-order autoregressive Beta-Gamma process, denoted as BGAR(1) \citep{lewis1989gamma}, whose stationary distribution is Gamma. To the best of our knowledge, this particular chain has never been considered to model temporal dependencies in matrix factorization problems;
	\item We derive majorization-minimization-based algorithms for maximum a posteriori (MAP) estimation in the NMF models (with a Poisson likelihood) with four of the presented prior structures on $\mathbf{H}$, including BGAR(1);
	\item We compare the performance of all these models on a prediction task on three real-world datasets.
\end{itemize}

The paper is organized as follows. Section~\ref{sec-study} introduces and compares non-negative Markov chains from the literature. Section~\ref{sec-map} presents MAP estimation in temporal NMF models. Experimental work is conducted in Section~\ref{sec-exp}, before concluding in Section~\ref{sec-ccl}.

\section{Comparative study of Gamma Markov chains} \label{sec-study}

This section reviews existing models of Gamma Markov chains, i.e., Markov chains which evolve in $\mathbb{R}_{+}$ in relation with the Gamma distribution. We have identified four different models in the NMF literature:
\begin{enumerate}
	\item Chaining on the rate parameter of a Gamma distribution (Section~\ref{sec-rate});
	\item Chaining on the rate parameter of a Gamma distribution with an auxiliary variable (Section~\ref{sec-hier-rate});
	\item Chaining on the shape parameter of a Gamma distribution (Section~\ref{sec-shape});
	\item Chaining on the shape parameter of a Gamma distribution with an auxiliary variable (Section~\ref{sec-hier-shape}).
\end{enumerate}
As will be discussed in these subsections, these four models all lack a well-defined stationary distribution, which leads to the degeneracy of the realizations of the chains. A fifth model from the time series literature, called BGAR(1), is presented in Section~\ref{sec-bgar}. It is built to have a well-defined stationary distribution (it is marginally Gamma distributed). The realizations of the chain are not degenerate and exhibit some interesting properties. To the best of our knowledge, this kind of process has never been used in a probabilistic NMF problem to model temporal evolution.

Throughout the section, $(h_n)_{n \geq 1}$ denotes the (scalar) Markov chain of interest, where the index $k$ as in Eq.~\eqref{eq-h-markov-idp} has been dropped for enhanced readability. It is further assumed that $h_1$ is set to a fixed, deterministic value.

\subsection{Chaining on the rate parameter} \label{sec-rate}

\subsubsection{Model} Let us consider a general Gamma Markov chain model with a chaining on the rate parameter:
\begin{equation}
h_{n}|h_{n-1} \sim \text{Gamma} \left( \alpha, \frac{\beta}{h_{n-1}} \right).
\label{eq-gmc-rate-def}
\end{equation}
As it turns out, Eq.~\eqref{eq-gmc-rate-def} can be rewritten as a multiplicative noise model:
\begin{equation}
h_{n} = h_{n-1} \times \phi_n,
\label{eq-gmc-rate-noise}
\end{equation}
where $\phi_n$ are i.i.d. Gamma random variables with parameters $(\alpha, \beta)$. We have
\begin{equation}
\mathbb{E}(h_{n}|h_{n-1}) = \frac{\alpha}{\beta}h_{n-1}, \quad
\text{var}(h_{n}|h_{n-1}) = \frac{\alpha}{\beta^2}h^2_{n-1}.
\end{equation}

This model was introduced in \citet{fevotte2009nonnegative} to add smoothness to the activation coefficients in the context of audio signal processing. The parameters were set to $\alpha > 1$ and $\beta = \alpha - 1$, such that the mode would be located at $h_{n} = h_{n-1}$. It is also a particular case of the dynamical model of \citet{fevotte2013non} and is considered in \citet{virtanen2020dynamic}. A similar inverse Gamma Markov chain was also considered in \citet{fevotte2009nonnegative} and in \citet{fevotte2011majorization}.

\subsubsection{Analysis} From Eq.~\eqref{eq-gmc-rate-noise} we can write:
\begin{equation}
h_{n} = h_1 \prod_{i=2}^{n} \phi_i.
\end{equation}

The independence of the $\phi_i$ yields:
\begin{align}
\mathbb{E}(h_{n}) & = h_1 \left( \frac{\alpha}{\beta} \right)^{n-1}, \\
\text{var}(h_{n}) & = h_1^2 \left[ \left( \frac{\alpha^2}{\beta^2} + \frac{\alpha}{\beta^2} \right)^{n-1} - \left( \frac{\alpha^2}{\beta^2} \right)^{n-1} \right].
\label{eq-gmc-rate-moments}
\end{align}

We enumerate all the possible regimes ($n \rightarrow +\infty)$, which all give rise to degenerate stationary distributions for different reasons:
\begin{itemize}
	\item $\beta > \sqrt{\alpha(\alpha+1)}$: both mean and variance go to zero;
	\item $\beta = \sqrt{\alpha(\alpha+1)}$: variance converges to 1, however the mean goes to zero;
	\item $\beta \in \left]\alpha;\sqrt{\alpha(\alpha+1)}\right[$: variance goes to infinity, mean goes to zero;
	\item $\beta = \alpha$: mean is equal to 1, but the variance goes to infinity;
	\item $\beta < \alpha$: both mean and variance go to infinity.
\end{itemize}

Each subplot of Figure~\ref{fig-rate} displays ten independent realizations of the chain, for a different set of parameters $(\alpha,\beta)$. As we can see, the realizations of the chain either collapse to 0, or diverge.

\subsection{Hierarchical chaining on the rate parameter} \label{sec-hier-rate}

\subsubsection{Model} Let us consider the following Gamma Markov chain model introduced in \citet{cemgil2007conjugate}:
\begin{align}
z_{n}|h_{n-1} & \sim \text{Gamma}(\alpha_z, \beta_z h_{n-1}), \label{eq-gmc-cd-def1} \\
h_{n}|z_{n} & \sim \text{Gamma}(\alpha_h, \beta_h z_{n}). \label{eq-gmc-cd-def2}
\end{align}
As it turns out, this model can also be rewritten as a multiplicative noise model:
\begin{equation}
h_{n} = h_{n-1} \times \tilde{\phi}_n,
\label{eq-gmc-cd-noise}
\end{equation}
where $\tilde{\phi}_n$ are i.i.d. random variables defined as the ratio of two independent Gamma random variables with parameters $(\alpha_h, \beta_h)$ and $(\alpha_z, \beta_z)$. The distribution of $\tilde{\phi}_n$ is actually known in closed form, namely
\begin{equation}
\tilde{\phi}_n \sim \text{BetaPrime} \left( \alpha_h, \alpha_z, 1, \tilde{\beta} \right),
\label{eq-gmc-cd-betaprime}
\end{equation}
with $\tilde{\beta} = \frac{\beta_z}{\beta_h}$ (see Appendix~\ref{app-a} for a definition). We have
\begin{align}
\mathbb{E}(h_{n}|h_{n-1}) & = \tilde{\beta} \frac{\alpha_h}{\alpha_z-1} h_{n-1} & \text{for~} \alpha_z > 1, \\
\text{var}(h_{n}|h_{n-1}) & = \tilde{\beta}^2 \frac{\alpha_h ( \alpha_h + \alpha_z - 1)}{(\alpha_z-1)^2 (\alpha_z-2)} h^2_{n-1} & \text{for~} \alpha_z > 2.
\end{align}

This model is less straightforward in its construction than the previous one, as it makes use of an auxiliary variable $z_n$ (note that a similar inverse Gamma construction was proposed as well in \citet{cemgil2007conjugate}). There are two motivations behind the introduction of this auxiliary variable:
\begin{enumerate}
	\item Firstly, it ensures what is referred to as ``positive correlation'' in \citet{cemgil2007conjugate}, i.e., $\mathbb{E}(h_n|h_{n-1}) \propto h_{n-1}$ (something the model described by Eq.~\eqref{eq-gmc-rate-def} does as well).
	\item Secondly, it ensures the so-called conjugacy of the model, i.e., the conditional distributions $p(z_{n}|h_{n-1},h_{n})$ and $p(h_{n}|z_{n},z_{n+1})$ remain Gamma distributions. Indeed, these are the distributions of interest when considering Gibbs sampling or variational inference. This property is not satistfied by the model described by Eq.~\eqref{eq-gmc-rate-def} (i.e., $p(h_n|h_{n-1},h_{n+1})$ is neither Gamma, nor a known distribution).
\end{enumerate}
This particular chain has been used in the context of audio signal processing in \citet{virtanen2008bayesian} (under the assumption of a Poisson likelihood, which does not fit the nature of the data), and also to model the evolution of user and item preferences in the context of recommender systems \citep{jerfel2017dynamic,do2018gamma}.

\subsubsection{Analysis} From Eq.~\eqref{eq-gmc-cd-noise}, we can write:
\begin{equation}
h_{n} = h_1 \prod_{i=2}^{n} \tilde{\phi}_{i}.
\end{equation}

We have by independence of the $\tilde{\phi}_i$:
\begin{align}
\mathbb{E}(h_{n}) & = h_{1} \left( \tilde{\beta} \frac{\alpha_h}{\alpha_z - 1} \right)^{n-1} \qquad \qquad \quad \text{for~}\alpha_z > 1,  \\
\text{var}(h_{n}) & = h_1^2 \tilde{\beta}^{2(n-1)} \left[
\left( \frac{\alpha_h^2}{(\alpha_z - 1)^2} + \frac{\alpha_h ( \alpha_h + \alpha_z - 1)}{(\alpha_z-1)^2 (\alpha_z-2)} \right)^{n-1} \right. \notag \\
& \left. \qquad \qquad \quad - \left( \frac{\alpha_h^2}{(\alpha_z - 1)^2} \right)^{n-1}
\right]~\text{for~} \alpha_z > 2.
\end{align}

As in the previous model, we can show that either the expectation or the variance diverges or collapses as $n \rightarrow \infty$ for every possible choice of parameters, which means that they all give rise to a degenerate stationary distribution of the chain. Each subplot of Figure~\ref{fig-rate-hier} displays ten independent realizations of the chain, for a different set of parameters $(\alpha_z,\beta_z,\alpha_h,\beta_h)$. As we can see, the realizations of the chain either collapse to~$0$ or diverge.

\subsection{Chaining on the shape parameter} \label{sec-shape}

\subsubsection{Model} Let us consider a general Gamma Markov chain model with a chaining on the shape parameter:
\begin{equation}
h_{n}|h_{n-1} \sim \text{Gamma}(\alpha h_{n-1}, \beta).
\label{eq-gmc-shape-def}
\end{equation}
We have
\begin{equation}
\mathbb{E}(h_{n}|h_{n-1}) = \frac{\alpha}{\beta}h_{n-1}, \quad
\text{var}(h_{n}|h_{n-1}) = \frac{\alpha}{\beta^2}h_{n-1}.
\end{equation}

In contrast with the two models presented previously, this model cannot be rewritten as a multiplicative noise model. This model is therefore more intricate to interpret. It was introduced in \citet{acharya2015nonparametric} in the context of Poisson factorization. It is mainly motivated by a data augmentation trick that can be used when working with a Poisson likelihood, which enables a Gibbs sampling procedure. The authors set the value of $\alpha$ to 1 (although the same trick can be applied for any value of $\alpha$). This model is also a particular case of the dynamical model of \citet{schein2016poisson}. It has since been used in the context of topic modeling \citep{acharya2018dual}.

\subsubsection{Analysis} Using the law of total expectation and total variance, it can be shown that
\begin{equation}
\mathbb{E}(h_{n}) = h_1 \left( \frac{\alpha}{\beta} \right)^{n-1},~
\text{var}(h_{n}) = h_{1} \frac{1}{\beta} \left( \frac{\alpha}{\beta} \right)^{n-1} \sum_{i=0}^{n-2} \left( \frac{\alpha}{\beta} \right)^i.
\end{equation}

The discussion is hence driven by the value of $r = \alpha/\beta$.
\begin{itemize}
	\item If $r < 1$, mean and variance go to zero;
	\item If $r = 1$, mean is fixed but variance goes to infinity (linearly);
	\item If $r > 1$, mean and variance go to infinity.
\end{itemize}

This chain only exhibits degenerate stationary distributions. Each subplot of Figure~\ref{fig-shape} displays ten independent realizations of the chain, for a different set of parameters $(\alpha,\beta)$. As we can see, the realizations of the chain either collapse to 0, or diverge.

\subsection{Hierarchical chaining on the shape parameter} \label{sec-hier-shape}

\subsubsection{Model}

Let us consider the following Gamma Markov chain model
\begin{align}
z_{n}|h_{n-1} & \sim \text{Poisson}(\beta h_{n-1}), \label{eq-gmc-hs-d1} \\
h_{n}|z_{n} & \sim \text{Gamma}(\alpha + z_n, \beta). \label{eq-gmc-hs-d2}
\end{align}
This model is a particular case of the dynamical model firstly introduced in \citet{schein2019poisson}. It cannot be rewritten as a multiplicative noise model. Using the law of total expectation and total variance, we obtain
\begin{align}
\mathbb{E}(h_{n}|h_{n-1}) & = h_{n-1} + \frac{\alpha}{\beta}, \\
\text{var}(h_{n}|h_{n-1}) & = \frac{2}{\beta} h_{n-1} + \frac{\alpha}{\beta^2}.
\end{align}
The motivation behind the introduction of this model is once again computational: it leads to closed-form conditional distributions when considering a Poisson likelihood. As stated in \citet{schein2019poisson}, the auxiliary variable $z_n$ can actually be marginalized out, leading to the so-called randomized Gamma distribution of the first type (RG1), whose analytical expression makes use of modified Bessel functions.

The authors also consider the particular limit case $\alpha = 0$, which leads here to a chain which will only take value 0 after obtaining $z_n = 0$.

\subsubsection{Analysis}

Using the law of total expectation and total variance, it can be shown that
\begin{align}
\mathbb{E}(h_{n}) & = h_1 + (n-1)\frac{\alpha}{\beta}, \\
\text{var}(h_{n}) & = (n-1)\frac{2}{\beta} h_{1} + (n-1)^2 \frac{\alpha}{\beta^2}.
\end{align}
As such, for $\alpha, \beta > 0$, when $n \rightarrow + \infty$, both the expectation and variance of $h_n$ diverge, leading to a degenerate stationary distribution of the chain. Each subplot of Figure~\ref{fig-shape-hier} displays ten independent realizations of the chain, for a different set of parameters $(\alpha,\beta)$.

\subsection{BGAR(1)} \label{sec-bgar}

We now discuss the first order autoregressive Beta-Gamma process of \citet{lewis1989gamma}, a stochastic process which is marginally Gamma distributed. The authors referred to the process as ``BGAR(1)''. However, to the best of our knowledge, no extension to higher-order autoregressive processes exists in the time series literature. As such, from now on, we will simply refer to it as ``BGAR''.

\subsubsection{Model}

Consider $\alpha > 0$, $\beta > 0$, $\rho \in [0,1[$. The BGAR process is defined as:
\begin{align}
h_1 & \sim \text{Gamma}(\alpha,\beta), \label{eq-gmc-bgar-def1} \\
h_n & = b_n h_{n-1} + \epsilon_n \qquad \text{for~} n \geq 2 \label{eq-gmc-bgar-def2},
\end{align}
where $b_n$ and $\epsilon_n$ are i.i.d. random variables distributed as:
\begin{align}
b_n & \sim \text{Beta}(\alpha \rho, \alpha(1-\rho)), \label{eq-gmc-bgar-def3} \\
\epsilon_n & \sim \text{Gamma}(\alpha(1-\rho), \beta) \label{eq-gmc-bgar-def4}.
\end{align}
The sequence $(h_n)_{n \geq 1}$ is called the BGAR process. It is parametrized~by $\alpha$, $\beta$ and $\rho$. We emphasize that the distribution $p(h_n|h_{n-1})$ is not known in closed form. Only $p(h_n|h_{n-1}, b_n)$ is known; it is a shifted Gamma distribution. The generative model may therefore be rewritten as
\begin{align}
h_1 & \sim \text{Gamma}(\alpha,\beta), \label{eq-b1} \\
b_n & \sim \text{Beta}(\alpha \rho, \alpha(1-\rho)) \quad \text{for~}n \geq 2, \label{eq-b2} \\
h_n | b_n, h_{n-1} & \sim \text{Gamma}(\alpha(1-\rho), \beta, \text{loc} = b_n h_{n-1}) \label{eq-b3} \\
& \qquad \qquad \qquad \qquad \qquad~\text{for~}n \geq 2, \notag
\end{align}
where the distribution in Eq.~\eqref{eq-b3} is a shifted Gamma distribution with a location parameter ``loc''.

We have
\begin{align}
\mathbb{E}(h_n|h_{n-1}) & = \rho h_{n-1} + \frac{\alpha(1-\rho)}{\beta}, \label{eq-gmc-bgar-cond-moment1} \\
\text{var}(h_n|h_{n-1}) & = \frac{\rho(1-\rho)}{\alpha + 1}h^2_{n-1} + \frac{\alpha(1-\rho)}{\beta^2}. \label{eq-gmc-bgar-cond-moment2}
\end{align}

\subsubsection{Analysis}

To study the marginal distribution of the process, we recall the following lemma.

\begin{mylem}
If $X \sim \text{Beta}(a,b)$ and $Y \sim \text{Gamma}(a+b,c)$ are independent random variables, then $Z = XY$ is $\text{Gamma}(a,c)$ distributed.
\label{lemma-1}
\end{mylem}

\begin{myprp}
Let $(h_n)_{n \geq 1}$ be a BGAR process. Then $h_n$ is marginally $\text{Gamma}(\alpha,\beta)$ distributed.
\end{myprp}

\begin{proof}
Follows by induction. Consider $n$ such that $h_n$ is $\text{Gamma}(\alpha,\beta)$ distributed. Then, $\epsilon_{n+1}h_n$ is $\text{Gamma}(\alpha \rho,\beta)$ distributed (Lemma~\ref{lemma-1}). Finally, $h_{n+1} = \epsilon_{n+1}h_n + b_{n+1}$ is $\text{Gamma}(\alpha,\beta)$ distributed (sum of independent Gamma random variables), which concludes the proof.
\end{proof}

Therefore the parameters $\alpha$ and $\beta$ control the marginal distribution. The parameter $\rho$ controls the correlation between successive values, as discussed in the following proposition.

\begin{myprp}
Let $(h_n)_{n \geq 1}$ be a BGAR process. Let $n$ and $r$ be two integers such that $r > 1$. We have $\text{corr}(h_n, h_{n+r}) = \rho^{r}$.
\label{prp-2}
\end{myprp}

\begin{proof}
See Appendix~\ref{app-b} for $r=1$.
\end{proof}

Proposition~\ref{prp-2} implies that the BGAR(1) process admits a (second order) AR(1) representation. Two limit cases of BGAR can be exhibited:
\begin{itemize}
	\item When $\rho = 0$, the $h_n$ are i.i.d. random variables;
	\item When $\rho \rightarrow 1$, the process is not random anymore, and $h_n = h_1$ for all $n$ (note that $\rho = 1$ is not an admissible value).
\end{itemize}

Finally, from Eq.~\eqref{eq-gmc-bgar-cond-moment1}, we have
\begin{equation}
\bigg( \mathbb{E}(h_n|h_{n-1}) > h_{n-1} \bigg) \Leftrightarrow \bigg( h_{n-1} < \frac{\alpha}{\beta} \bigg).
\end{equation}
If $h_{n-1}$ is below the mean of the marginal distribution $\mathbb{E}(h_n)=\frac{\alpha}{\beta}$, then $h_n$ will be in expectation above $h_{n-1}$, and vice-versa.

Note that BGAR is not the only Markovian process with a marginal Gamma distribution considered in the literature. We mention the GAR(1) process (first-order autoregressive Gamma process) of \citet{gaver1980first}, which is also marginally Gamma distributed. However, this particular process is piecewise deterministic, and its parameters are ``coupled'': the parameters of the marginal distribution also have an influence on other properties of the model. As such, it is less suited to our problem, and will not be considered here.

Figure~\ref{fig-bgar} displays three realizations of the BGAR process, with parameters fixed to $\alpha = 2$ and $\beta = 1$, and a different parameter $\rho$ in each subplot. The mean of the marginal distribution is displayed in red. When $\rho = 0.5$, the correlation is weak, and no particular structure is observed. However, as $\rho$ goes to 1, the correlation becomes stronger, and we typically observe piecewise constant trajectories.

\begin{figure}[t]
    \centering
    \includegraphics[width = 14cm]{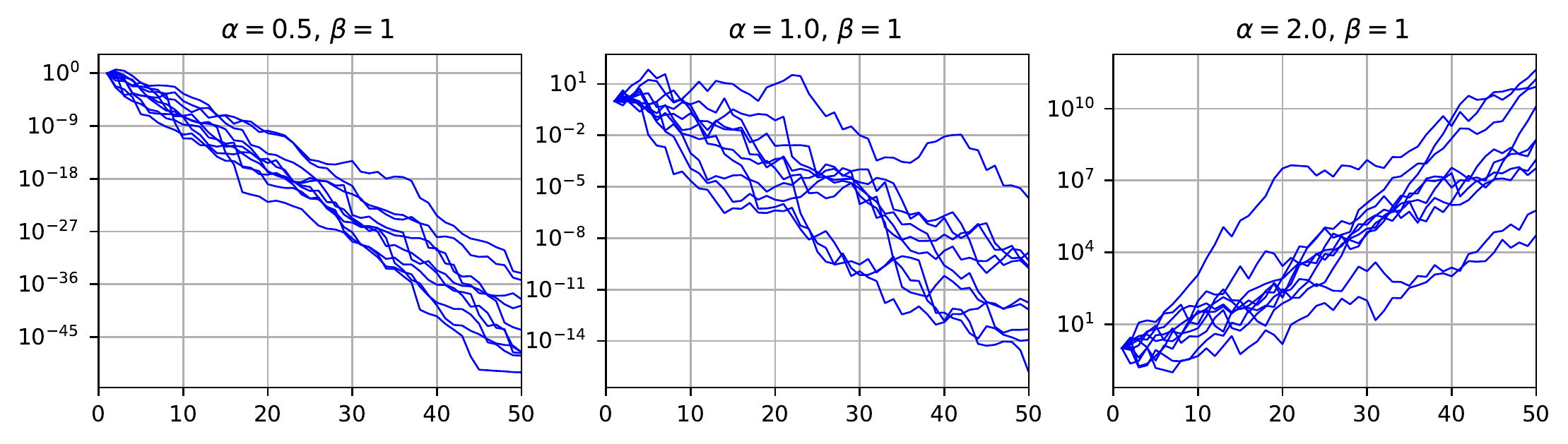}
    \caption{Realizations of the Markov chain defined in Eq.~\eqref{eq-gmc-rate-def}. The initial value $h_1$ is set to 1, and chains were simulated until $n=50$. Each subplot contains ten independent realizations, with the value of the parameters $(\alpha,\beta)$ given at the top of the subplot. $\log_{10}(h_n)$ is displayed.}
    \label{fig-rate}
\end{figure}

\begin{figure}[t]
    \centering
    \includegraphics[width = 14cm]{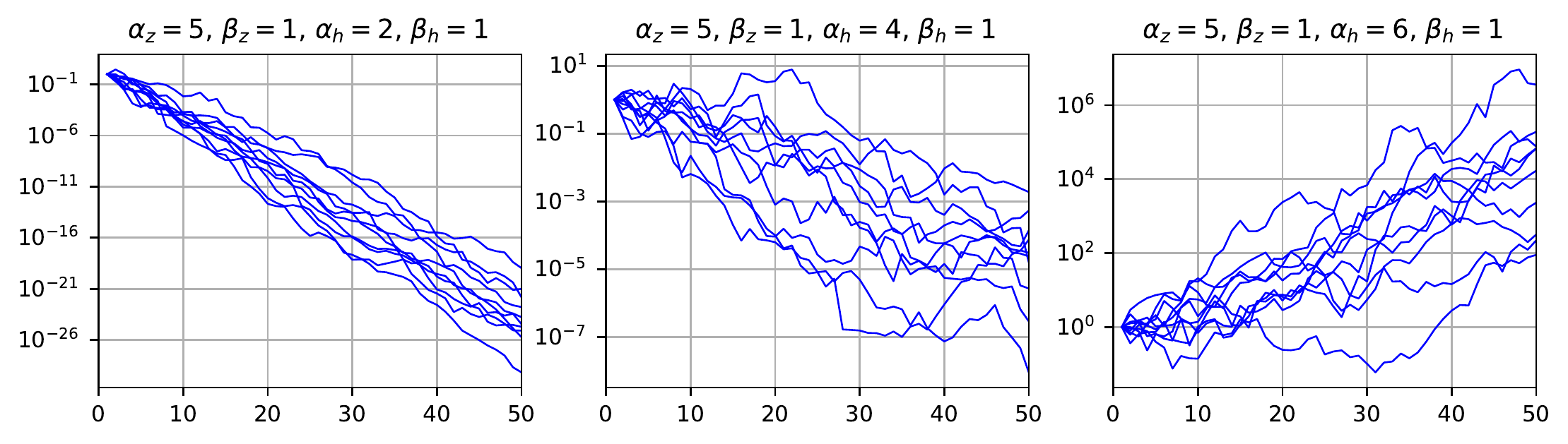}
    \caption{Realizations of the Markov chain defined in Eq.~\eqref{eq-gmc-cd-def1}-\eqref{eq-gmc-cd-def2}. The initial value $h_1$ is set to 1, and chains were simulated until $n=50$. Each subplot contains ten independent realizations, with the value of the parameters $(\alpha_z,\beta_z,\alpha_h,\beta_h)$ given at the top of the subplot. $\log_{10}(h_n)$ is displayed.}
    \label{fig-rate-hier}
\end{figure}

\begin{figure}[t]
    \centering
    \includegraphics[width = 14cm]{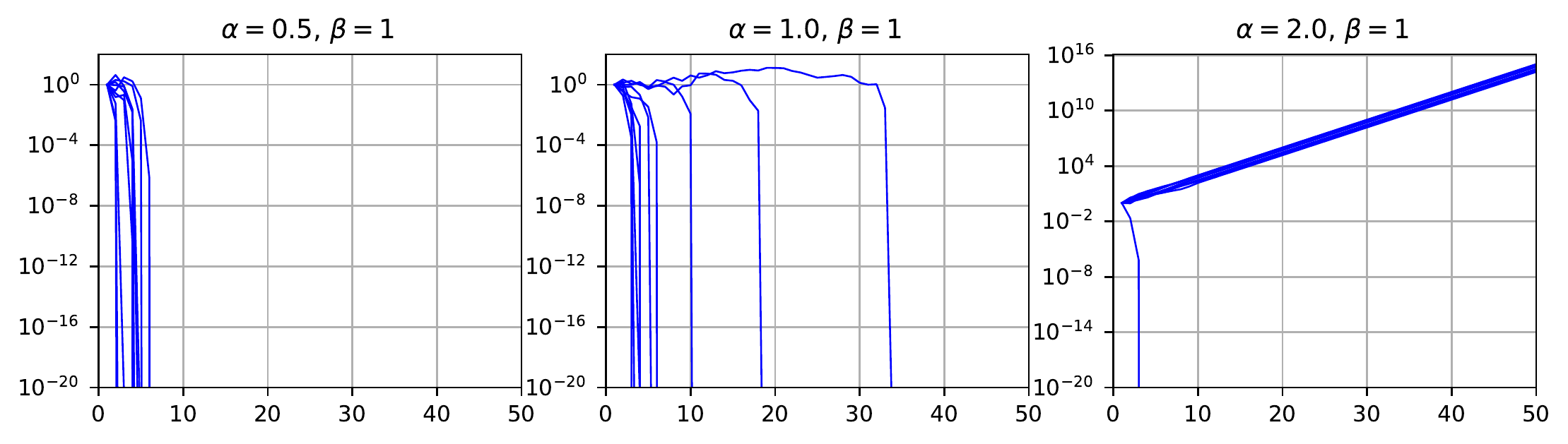}
    \caption{Realizations of the Markov chain defined in Eq.~\eqref{eq-gmc-shape-def}. The initial value $h_1$ is set to 1, and chains were simulated until $n=50$. Each subplot contains ten independent realizations, with the value of the parameters $(\alpha,\beta)$ given at the top of the subplot. $\log_{10}(h_n)$ is displayed.}
    \label{fig-shape}
\end{figure}

\begin{figure}[t]
    \centering
    \includegraphics[width = 14cm]{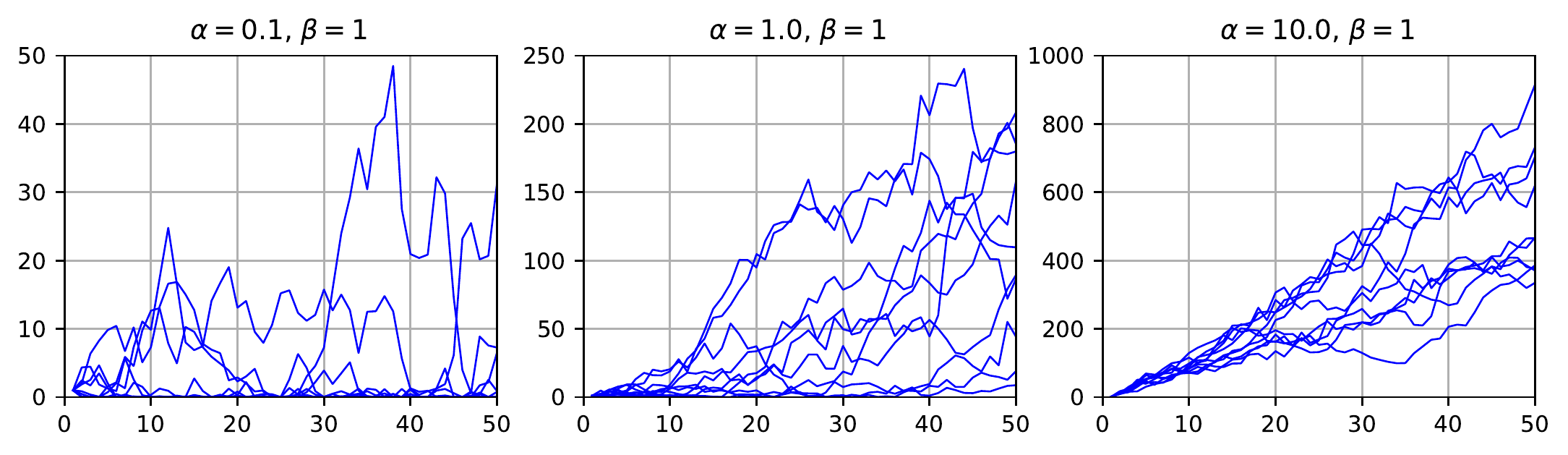}
    \caption{Realizations of the Markov chain defined in Eq.~\eqref{eq-gmc-hs-d1}-\eqref{eq-gmc-hs-d2}. The initial value $h_1$ is set to 1, and chains were simulated until $n=50$. Each subplot contains ten independent realizations, with the value of the parameters $(\alpha,\beta)$ given at the top of the subplot. $\log_{10}(h_n)$ is displayed.}
    \label{fig-shape-hier}
\end{figure}

\begin{figure}[t]
    \centering
    \includegraphics[width = 14cm]{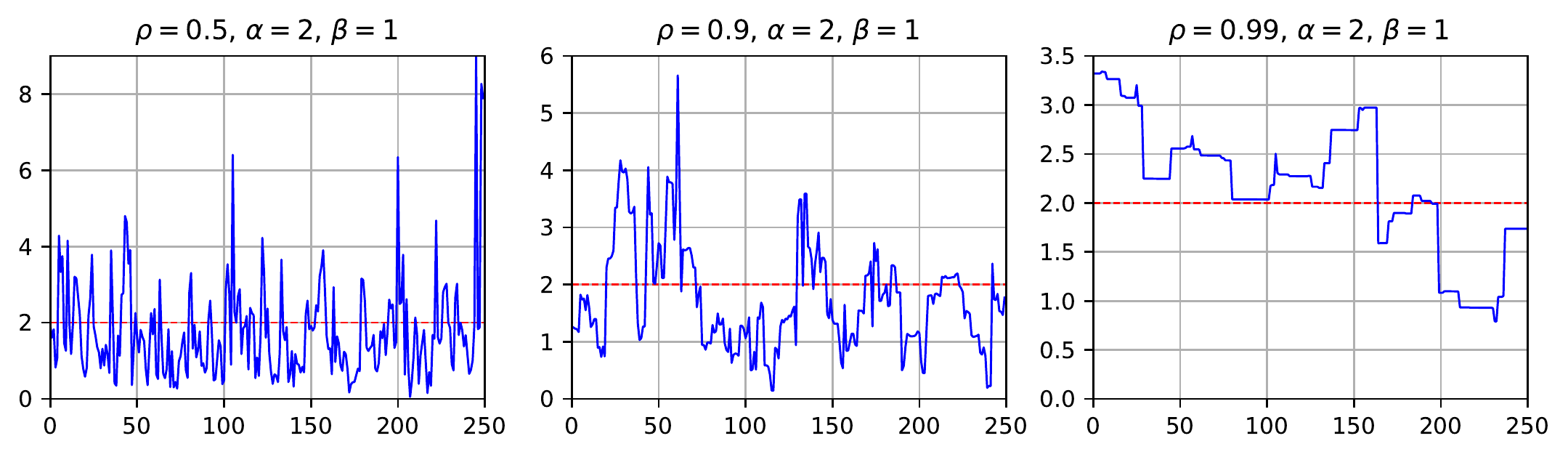}
    \caption{Three realizations of the BGAR(1) process, with parameters fixed to $\alpha = 2$ and $\beta = 1$, and a different parameter $\rho$ in each subplot. The mean of the process is displayed by a dashed red line.}
    \label{fig-bgar}
\end{figure}

\clearpage

\section{MAP estimation in temporal NMF models} \label{sec-map}

We now turn to the problem of maximum a posteriori (MAP) estimation in temporal NMF models. More precisely, we assume a Poisson likelihood, that is
\begin{equation}
v_{fn} \sim \text{Poisson}([\mathbf{WH}]_{fn}),
\end{equation}
and we also assume that $\mathbf{W}$ is a deterministic variable. The variables $\mathbf{V}$ and $\mathbf{H}$ then define a hidden Markov model, as displayed on Figure~\ref{figure-hmm}.

We consider four different models corresponding to the temporal structures on $\mathbf{H}$ presented in subsections \ref{sec-rate}, \ref{sec-hier-rate}, \ref{sec-shape}, and \ref{sec-bgar}. Only the temporal structure presented in~\ref{sec-hier-shape} is left out. Indeed, deriving a MAP algorithm in this model using the auxiliary variables $\mathbf{Z}$ (similar to the one of Section~\ref{sec-alg-map-hier}) would involve integer programming. This leads to technical developments which are out-of-scope of our current study.

\begin{figure}[h]
\centering
	\begin{tikzpicture}[scale = 1.25]
	\draw (0,0) circle (0.5);
	\draw (1.5,0) circle (0.5);
	\draw (3,0) circle (0.5);
	\draw [fill=blue!25](0,-1.5) circle (0.5);
	\draw [fill=blue!25](1.5,-1.5) circle (0.5);
	\draw [fill=blue!25](3,-1.5) circle (0.5);
	\draw [->] (0.5,0) -- (1,0);
	\draw [->] (2,0) -- (2.5,0);
	\draw [<-] (0, -1) -- (0, -0.5);
	\draw [<-] (1.5, -1) -- (1.5, -0.5);
	\draw [<-] (3, -1) -- (3, -0.5);
	\draw [dashed] (3.5, 0) -- (3.95, 0);
	\draw [dashed] (-0.5, 0) -- (-0.95, 0);
	\draw (0,0) node{{$\mathbf{h}_{n-1}$}};
	\draw (1.5,0) node{{$\mathbf{h}_{n}$}};
	\draw (3,0) node{{$\mathbf{h}_{n+1}$}};
	\draw (0,-1.5) node{{$\mathbf{v}_{n-1}$}};
	\draw (1.5,-1.5) node{{$\mathbf{v}_{n}$}};
	\draw (3,-1.5) node{{$\mathbf{v}_{n+1}$}};

	\draw (1.5,-2.5) node {$\bullet$};
	\draw (1.5,-2.8) node {{$\mathbf{W}$}};
	\draw [->] (1.5,-2.5) -- (1.5,-2);
	\draw [->] (1.5,-2.5) -- (0,-2);
	\draw [->] (1.5,-2.5) -- (3,-2);
	\draw [dashed] (1.5, -2.5) -- (-0.75, -2.125);
	\draw [dashed] (1.5, -2.5) -- (3.75, -2.125);
	\end{tikzpicture}
	\caption{Hidden Markov model arising in temporal NMF models. $\mathbf{v}_n$ is of dimension $F$, while $\mathbf{h}_n$ is of dimension $K$. Observed variables are in blue.}
	\label{figure-hmm}
\end{figure}
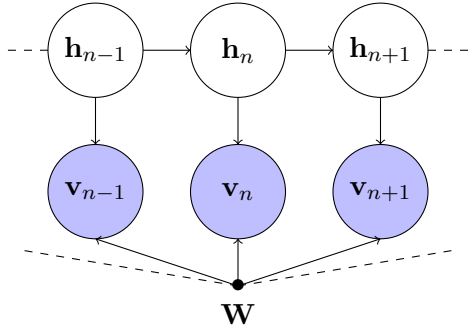

Generally speaking, joint MAP estimation in such models amounts to minimizing the following criterion
\begin{align}
C(\mathbf{W},\mathbf{H}) & = -\log p(\mathbf{V},\mathbf{H};\mathbf{W}) \label{eq-ccc} \\
& = -\log p(\mathbf{V}|\mathbf{H};\mathbf{W}) - \sum_k \left[ \log p(h_{k1}) + \sum_{n \geq 2} \log p(h_{kn}|h_{k(n-1)}) \right],
\end{align}
that is to say that the factors $\mathbf{W}$ and $\mathbf{H}$ are going to be estimated. Both shape hyperparameters ($\alpha_k$ or $\rho_k$) and scale hyperparameters ($\beta_k$) will be treated as fixed and selected using a validation set. However, note that deriving the maximum likelihood estimate of $\beta_k$ is feasible in closed form for all the models presented below. Unfortunately, estimating $\beta_k$ this way led to overly flat priors in our experience (likely due to the MAP estimation setting).

The optimization of the function $C$ is carried out with a block coordinate descent scheme over the variables $\mathbf{W}$ and~$\mathbf{H}$. We resort to a majorization-minimization (MM) scheme, which consists in iteratively majorizing the function $C$ (by a so-called auxiliary function, tight for some $\tilde{\mathbf{W}}$ or $\tilde{\mathbf{H}}$), and minimizing this auxiliary function instead. We refer the reader to \citet{hunter2004tutorial} for a detailed tutorial. Under this scheme, the function $C$ is non-increasing. As it turns out, only the Poisson likelihood term $-\log p(\mathbf{V}|\mathbf{H};\mathbf{W})$ needs to be majorized. This is a well-studied issue in the NMF literature. As stated in \citet{lee2000algorithms,fevotte2011algorithms}, the function
\begin{equation}
G_1(\mathbf{H};\tilde{\mathbf{H}}) = - \sum_{k,n} {p}_{kn} \log (h_{kn}) + \sum_{k,n} q_{k} h_{kn},
\label{eq-g1}
\end{equation}
with the notations
\begin{equation}
p_{kn} = \tilde{h}_{kn} \sum_f w_{fk} \frac{v_{fn}}{[\mathbf{W\tilde{H}}]_{fn}}, \quad q_k = \sum_f w_{fk},
\end{equation}
is a tight auxiliary function of $- \log p(\mathbf{V}|\mathbf{H};\mathbf{W})$ at $\mathbf{H} = \tilde{\mathbf{H}}$. Similarly the function
\begin{equation}
G_2(\mathbf{W};\tilde{\mathbf{W}}) = - \sum_{f,k} p'_{fk} \log (w_{fk}) + \sum_{f,k} q'_{k} w_{fk},
\label{eq-g2}
\end{equation}
with the notations
\begin{equation}
p'_{fk} = \tilde{w}_{fk} \sum_n h_{kn} \frac{v_{fn}}{[\mathbf{\tilde{W}H}]_{fn}}, \quad q'_k = \sum_n h_{kn},
\end{equation}
is a tight auxiliary function of $- \log p(\mathbf{V}|\mathbf{H};\mathbf{W})$ at $\mathbf{W} = \tilde{\mathbf{W}}$.

\subsection{Minimization w.r.t. \textbf{W}}

The optimization w.r.t. $\mathbf{W}$ is common to all algorithms, and amounts to minimizing $G_2(\mathbf{W};\tilde{\mathbf{W}})$ only. The scale of~$\mathbf{W}$ must be however be fixed in order to prevent potential degenerate solutions such that $\mathbf{W} \rightarrow + \infty$ and $\mathbf{H} \rightarrow 0$. Indeed, consider $\mathbf{W}^{\star}$ and $\mathbf{H}^{\star}$ minimizers of Eq.~\eqref{eq-ccc}, and let $\boldsymbol{\Lambda}$ be a diagonal matrix with non-negative entries. Then
\begin{align}
C(\mathbf{W}^{\star} \boldsymbol{\Lambda}^{-1}, \boldsymbol{\Lambda} \mathbf{H}^{\star})
& = - \log p(\mathbf{V}|\boldsymbol{\Lambda} \mathbf{H}^{\star}; \mathbf{W}^{\star} \boldsymbol{\Lambda}^{-1}) - \log p(\boldsymbol{\Lambda} \mathbf{H}^{\star}) \\
& = - \log p(\mathbf{V}|\mathbf{H}^{\star}; \mathbf{W}^{\star}) - \log p(\boldsymbol{\Lambda} \mathbf{H}^{\star}),
\end{align}
and depending on the choice of the prior distribution $p(\mathbf{H})$, we may obtain $C(\mathbf{W}^{\star} \boldsymbol{\Lambda}^{-1}, \boldsymbol{\Lambda} \mathbf{H}^{\star}) < C(\mathbf{W}^{\star}, \mathbf{H}^{\star})$, i.e., a contradiction. Therefore, in the following we impose that $||\mathbf{w}_k||_1 = 1$.

The constrained optimization is performed with the following update rule
\begin{equation}
w_{fk} = \frac{p'_{fk}}{\sum_f p'_{fk}}, \label{eq-upw}
\end{equation}
see Appendix~\ref{app-c} for the proof.

The following subsections detail the optimization w.r.t. $\mathbf{H}$ (and other variables when necessary) in the four considered models, which amounts to the minimization of $G_1(\mathbf{H};\tilde{\mathbf{H}}) - \log p(\mathbf{H})$.

\subsection{Chaining on the rate parameter} \label{sec-alg-map-rate}

The transition distribution $p(h_{kn}|h_{k(n-1)})$ is given by Eq.~\eqref{eq-gmc-rate-def}. The optimization w.r.t. $h_{kn}$ amounts to solving an order-2 polynomial equation on $\mathbb{R}_+$
\begin{equation}
a_{2,{kn}} h_{kn}^2 + a_{1,{kn}} h_{kn} + a_{0,{kn}} = 0.
\label{eq-poly-1}
\end{equation}
As it turns out, there is always exactly one non-negative root. The coefficients of the polynomial equation are given in Table~\ref{table-1}. This bears resemblance with the methodology described in \citet{fevotte2009nonnegative}, where the authors aimed at retrieving MAP estimates with a EM-like algorithm (with an exponential likelihood).

\begin{table}[t]
	\caption{Coefficients of the polynomial equation Eq.~\eqref{eq-poly-1}}
	\label{table-1}
	\centering
	\begin{tabular}{cccc}
	\toprule
	$n$ & $a_{2,{kn}}$ & $a_{1,{kn}}$ & $a_{0,{kn}}$ \\
	\midrule
	$1$ & $q_{k}$ & $\alpha_k - p_{1k}$ & $-\beta_k h_{k2}$ \\
	\\
	$2,\dotsc,N-1$ & $q_k + \frac{\beta_k}{h_{k(n-1)}}$ & $1 - p_{kn}$ & $- \beta_k h_{k(n+1)}$ \\
	\\
	$N$ & 0 & $q_k$ + $\frac{\beta}{h_{k(N-1)}}$ & $1- \alpha_k - p_N$ \\
	\toprule
	\end{tabular}
\end{table}

\subsection{Hierarchical chaining on the rate parameter} \label{sec-alg-map-hier}

In this case, we resort to using the auxiliary variables $\mathbf{Z}$, which results in the the slightly more involved following criterion
\begin{align}
C(\mathbf{W},\mathbf{H},\mathbf{Z}) = & -\log p(\mathbf{V}|\mathbf{H};\mathbf{W}) \\
& - \sum_k \left[ \log p(h_{k1}) + \sum_{n \geq 2} \left( \log p(z_{kn}|h_{k(n-1)}) + \log p(h_{kn}|z_{kn}) \right) \right] \notag.
\end{align}
We recall that $p(z_{kn}|h_{k(n-1)})$ and $p(h_{kn}|z_{kn})$ are given by Eq.~\eqref{eq-gmc-cd-def1} and Eq.~\eqref{eq-gmc-cd-def2}, respectively. Note that \citet{cemgil2007conjugate} proposed a Gibbs sampler and variational inference, and as such the development of the MAP algorithm is novel.

Imposing $\alpha_{h,k} \geq 1$, we obtain the following update for $z_{kn}$
\begin{equation}
z_{kn} = \frac{\alpha_{z,k} + \alpha_{h,k} - 1}{\beta_{z,k} h_{k(n-1)} + \beta_{h,k} h_{kn}},
\end{equation}
and the following updates for $h_{kn}$
\begin{align}
h_{k1} & = \frac{p_{k1} + \alpha_{z,k}}{q_k + \beta_{z,k} z_{k2}}, \\
h_{kn} & = \frac{p_{kn} + \alpha_{h,k} + \alpha_{z,k} - 1}{q_k + \beta_{h,k} z_{kn} + \beta_{z,k} z_{k(n+1)}}, n \in \{ 2,\dotsc, N-1 \}, \\
h_{kN} & = \frac{p_{kN} + \alpha_{h,k} - 1}{q_k + \beta_{h,k} z_{kN}}.
\end{align}

\subsection{Chaining on the shape parameter} \label{sec-alg-map-shape}

The transition distribution $p(h_{kn}|h_{k(n-1)})$ is given by Eq.~\eqref{eq-gmc-shape-def}. The optimization w.r.t. $h_{kn}$ amounts to solving the following equations on $\mathbb{R}_+$
\begin{equation}
-p_{k1} + (q_k - \alpha_k \log(\beta_k h_{k2}) + \alpha_k \Psi(\alpha h_{k1}) )h_{k1} = 0,
\end{equation}
\begin{align}
& (1-\alpha_k h_{k(n-1)}-p_{kn}) + (q_k + \beta_k - \alpha_k \log(\beta_k h_{k(n+1)})) h_{kn} + \alpha_k \Psi(\alpha_k h_{kn}) h_{kn} = 0,  \\
& \text{for } n \in \{ 2,\dotsc, N-1 \} \notag,
\end{align}
where $\Psi$ denotes the digamma function. Solving such equations can be done numerically with Newton's method. Finally the update for $h_{kN}$ is given by
\begin{equation}
h_{kN} = \frac{p_{kn} + \alpha_k h_{k(N-1)} -1}{q_k + \beta_k}.
\end{equation}

Note that a Gibbs sampling procedure is proposed in \citet{acharya2015nonparametric,schein2016poisson}, and as such the development of the MAP algorithm is novel.

\subsection{BGAR(1)} \label{sec-alg-map-bgar}

In this case, since the transition distribution $p(h_{kn}|h_{k(n-1)})$ is not known in closed form, we resort to optimizing the slightly more involved following criterion
\begin{align}
C(\mathbf{W},\mathbf{H},\mathbf{B}) = & -\log p(\mathbf{V}|\mathbf{H};\mathbf{W}) \label{eq-map-bgar} \\
& - \sum_k \left[ \log p(h_{k1}) + \sum_{n \geq 2} \left( \log p(h_{kn}|h_{k(n-1)},b_{kn}) + \log p(b_{kn}) \right) \right] \notag.
\end{align}

In the following, we will use the notations $\gamma_k = \alpha_k(1-\rho_k)$ and $\eta_k = \alpha_k \rho_k$.

\subsubsection{Constraints}

By construction, the variables $h_{kn}$ and $b_{kn}$ must lie in a specific interval given the values of all the other variables. Indeed, as $h_{kn} = b_{kn} h_{k(n-1)} + \epsilon_{kn}$ (see Eq.~\eqref{eq-gmc-bgar-def2}), where $\epsilon_{kn}$ is a non-negative random variable, we obtain ${h_{kn} \geq b_{kn} h_{k(n-1)}}$,  $b_{kn} \leq \frac{h_{kn}}{h_{k(n-1)}}$, and $h_{kn} \leq \frac{h_{k(n+1)}}{b_{k(n+1)}}$.

This leads to the following constraints
\begin{align}
0 & \leq  h_{k1} \leq \frac{h_{k2}}{b_{k2}}, \\
b_{kn} h_{k(n-1)} & \leq h_{kn} \leq \frac{h_{k(n+1)}}{b_{k(n+1)}} \quad n \in \{ 2,\dotsc, N-1 \}, \\
b_{kN} h_{k(N-1)} & \leq h_{kN},
\end{align}
and
\begin{equation}
0 \leq b_{kn} \leq \min \left( 1, \frac{h_{kn}}{h_{k(n-1)}} \right).
\end{equation}

We therefore introduce the notations
\begin{align}
c_{kn} = b_{kn} h_{k(n-1)}, \quad d_{kn} = \frac{h_{k(n+1)}}{b_{k(n+1)}}, \quad x_{kn} = \frac{h_{kn}}{h_{k(n-1)}},
\end{align}
as these quantities arise naturally in our derivations.

\subsubsection{Minimization w.r.t. $h_{kn}$}

\begin{table*}[t]
	\caption{Coefficients of the polynomial equation Eq.~\eqref{eq-poly-2}. Def. int. = Definition interval.}
	\label{table-2}
	\centering
	\resizebox{\textwidth}{!}{%
	\begin{tabular}{ccp{3.2cm}p{3.2cm}p{3.2cm}p{3.2cm}}
	\toprule
	$n$ & Def. interval & $a_{3,{kn}}$ & $a_{2,{kn}}$ & $a_{1,{kn}}$ & $a_{0,{kn}}$ \\
	\midrule
	$1$ & $[0, d_{k1}]$ & 0 & $-(q_k + \beta_k(1-b_{k2}))$ & $-(1-\alpha_k - p_{k1}) + (q_k + \beta_k(1-b_{k2})) d_{k1} - (1-\gamma_k)$ & $(1-\alpha_k - p_{k1}) d_{k1}$ \\
	\\
	$2,\dotsc,N-1$ & $[c_{kn}, d_{kn}]$ & $-(q_{k} + \beta_k(1-b_{k(n+1)}))$ & $p_{kn} - 2(1 - \gamma_k) + (q_{k} + \beta_k(1-b_{k(n+1)})\left( c_{kn} + d_{kn} \right)$ & $-p_{kn}\left( c_{kn} + d_{kn} \right) + (1- \gamma_k)\left( c_{kn} + d_{kn} \right) - (q_{k} + \beta_k(1-b_{k(n+1)})) c_{kn} d_{kn}$ & $p_{kn} c_{kn} d_{kn}$ \\
	\\
	$N$ & $[c_{kN},+\infty[$ & 0 & $q_k + \beta_k$ & $-p_{kN} - c_{kN}(q_k + \beta_k) + (1-\gamma_k)$ & $c_{kN} p_{kN}$ \\
	\toprule
	\end{tabular}
    }    	
\end{table*}

The optimization of Eq.~\eqref{eq-map-bgar} w.r.t. $h_{kn}$ may give rise to intractable problems, due to the logarithmic terms in the objective function. To alleviate this issue, we propose to control the limit values of the auxiliary function, by restricting ourselves to certain values of the hyperparameters. In particular, choosing $(1-\gamma_k) < 0$ ensures the existence of at least one minimizer.

For all $n$, the optimization w.r.t. $h_{kn}$ amounts to solving an order-3 polynomial equation
\begin{equation}
a_{3,{kn}} h_{kn}^3 + a_{2,{kn}} h_{kn}^2 + a_{1,{kn}} h_{kn} + a_{0,{kn}} = 0.
\label{eq-poly-2}
\end{equation}
The coefficients of the equation and definition intervals are given in Table~\ref{table-2}. If several roots belong to the definition interval, we simply choose the root which gives the lowest objective value.

\subsubsection{Minimization w.r.t. $b_{kn}$}

Similarly, logarithmic terms of the objective function may give rise to degenerate solutions. Using the same reasoning, we choose to impose ${(1-\gamma_k) < 0}$ and ${(1-\eta_k) < 0}$ to ensure the existence of at least one minimizer. 

The minimization of the auxiliary function w.r.t. $b_{kn}$ amounts to solving the following order 3 polynomial over the interval $[0, \min(1, x_{kn})]$
\begin{equation}
a_{3,kn} b_{kn}^3 + a_{2,kn} b_{kn}^2 + a_{1,kn} b_{kn} + a_{0,kn} d_{kn} = 0,
\end{equation}
where
\begin{align}
a_{3,kn} & = -\beta_k h_{k(n-1)}, \\
a_{2,kn} & = 2(1-\gamma_k) + (1-\eta_k) + \beta_k h_{k(n-1)}(x_{kn}+1), \\
a_{1,kn} & = -(1-\gamma_k)(x_{kn}+1) - (1-\eta_k)(x_{kn} + 1)  \\
& \quad - \beta_k h_{k(n-1)} x_{kn},\notag \\
a_{0,kn} & = (1-\eta_k) x_{kn}.
\end{align}

\subsubsection{Admissible values of hyperparameters}

To recap the discussion on admissible values of hyperparameters, to ensure the existence of minimizers of the auxiliary function, we have restricted ourselves to
\begin{equation}
\left\{
\begin{array}{l}
\alpha_k(1-\rho_k) > 1, \\
\alpha_k \rho_k > 1.
\end{array}
\right.
\end{equation}
This set is graphically displayed on Figure~\ref{fig-admv}. As we can see, choosing the value of $\rho_k$ to be close to one (to ensure correlation) leads to high values of $\alpha_k$.

\begin{figure}
	\centering
	\includegraphics[height = 5.5cm]{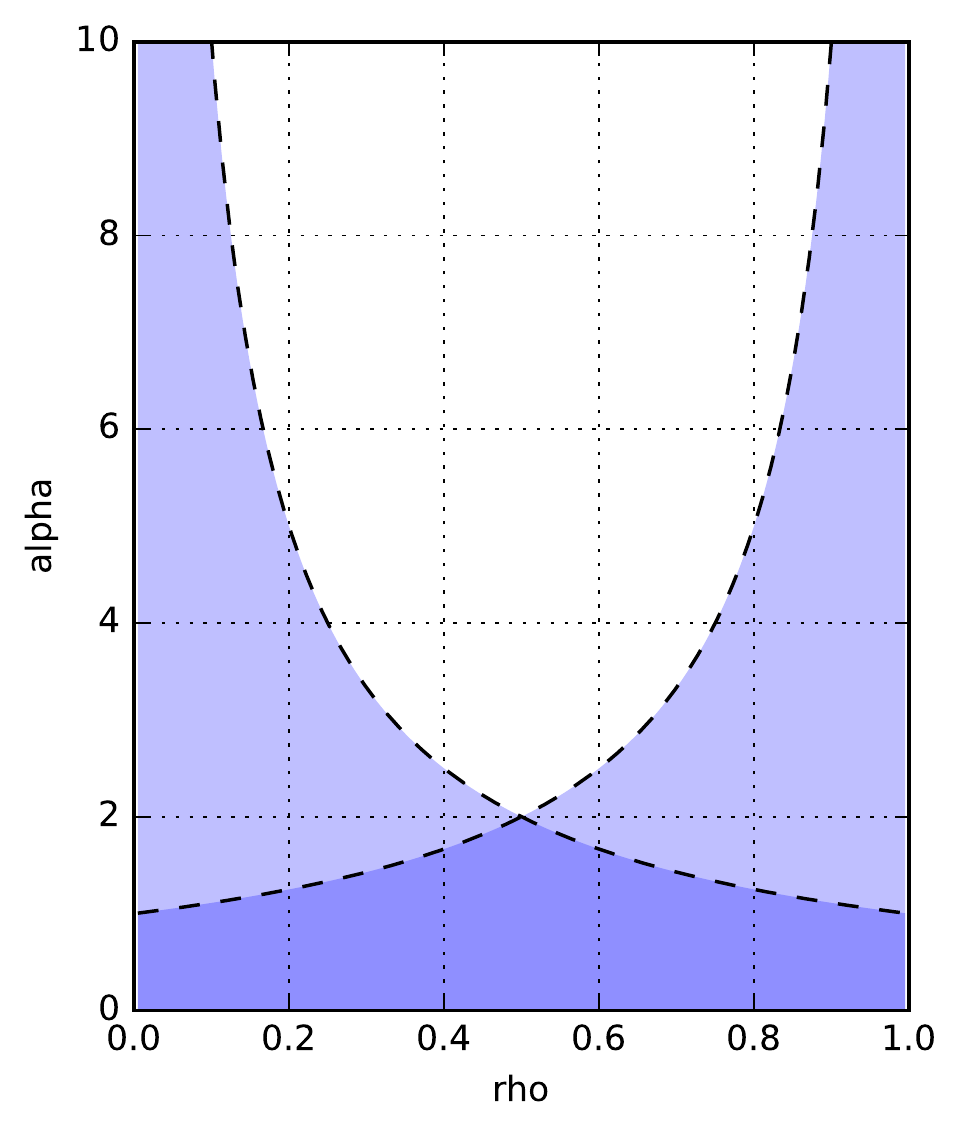}
	\caption{Admissible values of the hyperparameters in the MAP algorithm presented in Section~\ref{sec-alg-map-bgar}. Admissible values are in white.}
	\label{fig-admv}
\end{figure}

\section{Experimental work} \label{sec-exp}

We now compare the performance of all considered temporal NMF models on a prediction task on three real datasets. This task will consist in hiding random columns of the considered datasets and estimating those missing values. We will also include the performance of a naive baseline, which we detail in the following subsection. Adapting the MAP algorithms presented in Section~\ref{sec-study} in a setting with a mask of missing values only consist in a slight modification, presented in Appendix~\ref{app-d}. Python code is available online\footnote{\url{https://github.com/lfilstro/TemporalNMF}}.

\subsection{Experimental protocol}

For each considered dataset, the experimental protocol is as follows.

First of all, a value of the factorization rank $K$ (which will be used for all considered methods) must be selected. To do so, we apply the standard KL-NMF algorithm \citep{lee2000algorithms,fevotte2011algorithms} on 10 random training sets, which consist of $80 \%$ of the original data, with a pre-defined grid of values for $K$. We then select the value of $K$ which yields the lowest generalized Kullback-Leibler error (KLE) (see definition below) on the remaining $20 \%$ of the data, on average.

For the prediction experiment itself, we create 5 random splits of the data matrix, where $80 \%$ corresponds to the training set, $10 \%$ to the validation set, and the remaining $10 \%$ to the test set. To do so, we randomly select non-adjacent columns of the data matrix (excluding the first one and always including the last one), half of which will make up the validation set and the other half the test set (the last column is always included in the test set). We also consider 5 different random initializations.

Then, for each split-initialization pair, all the algorithms are run from this initialization point on the training set until convergence (the algorithms are stopped when the relative decrease of the objection function falls under $10^{-5}$). For each method, a grid of hyperparameters is considered, and their selection is based on the lowest KLE on the validation set. Details of the grids used for each method can be found in Appendix~\ref{app-e}. The predictive performance of each method is then computed on the test set by comparing the original value~$v_{fn}$ and its associated estimate $\hat{v}_{fn} = [\mathbf{WH}]_{fn}$ with the following metric. Denoting by $\mathcal{T}$ the test set, we compute the generalized Kullback-Leibler error (KLE), which is defined as
\begin{equation}
\text{KLE} = \sum_{(f,n) \in \mathcal{T}} \left[ v_{fn} \log \left( \frac{v_{fn}}{\hat{v}_{fn}} \right) - v_{fn} + \hat{v}_{fn} \right].
\end{equation}

Finally, we construct a baseline based on the Gamma-Poisson (GaP) model of \citet{canny2004gap}. The GaP model is based on independent Gamma priors on $\mathbf{H}$, i.e., a non-temporal prior. However, it is unable to estimate columns $\mathbf{h}_n$ associated with missing columns $\mathbf{v}_n$. We propose to set $\mathbf{h}_n = \frac{1}{2}(\mathbf{h}_{n-1} + \mathbf{h}_{n+1})$ and $\mathbf{h}_N = \mathbf{h}_{N-1}$ for these columns. MAP estimation in the GaP model is described in \citet{dikmen2012maximum} and is recalled in Appendix~\ref{app-f}.

\subsection{Datasets}

The following datasets are considered
\begin{itemize}
	\item The \texttt{NIPS} dataset\footnote{\url{https://archive.ics.uci.edu/ml/datasets/NIPS+Conference+Papers+1987-2015}}, which contains word counts (with stop words removed) of all the articles published at the NIPS\footnote{Now called NeurIPS.} conference between 1987 and 2015. We grouped the articles per publication year, yielding an observation matrix of size $11463 \times 29$. We obtained $K = 3$.
	\item The \texttt{last.fm} dataset, based on the so-called ``last.fm 1K'' users\footnote{\url{http://ocelma.net/MusicRecommendationDataset/}}, which contains the listening history with timestamps information of users of the music website last.fm. We preprocessed this dataset to obtain the monthly evolution of the listening counts of artists with at least 20 different listeners. This yields a dataset of size $7017 \times 53$ (i.e., we have the listening history of 7017 artists over 53 months). We obtained $K = 5$.
	\item The \texttt{ICEWS} dataset\footnote{\url{https://github.com/aschein/pgds}}, an international relations dataset, which contains the number of interactions between two countries for each day of the year 2003. The matrix is of size $6197 \times 365$. We obtained $K = 5$.
\end{itemize}

\subsection{Experimental results}

As previously mentioned, the test set consists of 10 \% of the columns of the data matrix $\mathbf{V}$, always including the last one. The KLE will be computed separately on all the columns minus the last one (denoted by "S" for smoothing), and on the last one (denoted by "F" for forecasting). Their averaged values over the 25 split-initialization pairs are reported on Table~\ref{table-nips} for the \texttt{NIPS} dataset, on Table~\ref{table-last} for the \texttt{last.fm} dataset, and on Table~\ref{table-icews} for the \texttt{ICEWS} dataset.

\begin{table*}[t]
	\centering
	\begin{tabular}{ccc}
        \toprule
        Model & KLE-S & KLE-F \\
        \midrule
        GaP (App.~\ref{app-f}) & $6.19\times 10^4 \pm 9.69\times 10^3$ & $1.08\times 10^5 \pm 2.67\times 10^3$ \\
        Rate~(\ref{sec-alg-map-rate}) & $6.07\times 10^4 \pm 8.97\times 10^3$ & $1.03\times 10^5 \pm 3.53\times 10^3$ \\
        Hier~(\ref{sec-alg-map-hier}) & $6.06\times 10^4 \pm 9.18\times 10^3$ & $3.24\times 10^5 \pm 2.09\times 10^5$ \\
        Shape~(\ref{sec-alg-map-shape}) & $9.37\times 10^4 \pm 2.99\times 10^4$ & $1.30\times 10^5 \pm 2.35\times 10^4$ \\
        BGAR~(\ref{sec-alg-map-bgar}) & $6.17\times 10^4 \pm 8.24\times 10^3$ & $1.36\times 10^5 \pm 2.00\times 10^3$ \\
        \toprule
	\end{tabular}
	\caption{Prediction results on the \texttt{NIPS} dataset. Lower values are better. The mean and standard deviation of each metric are reported over 25 runs.}
	\label{table-nips}
\end{table*}

\begin{table*}[t]
	\centering
	\begin{tabular}{ccc}
        \toprule
        Model & KLE-S & KLE-F \\
        \midrule
        GaP (App.~\ref{app-f}) & $1.30\times 10^4 \pm 3.35\times 10^2$ & $6.89\times 10^3 \pm 5.84\times 10^1$ \\
        Rate~(\ref{sec-alg-map-rate}) & $1.23\times 10^4 \pm 2.35\times 10^2$ & $7.76\times 10^3 \pm 4.13\times 10^1$ \\
        Hier~(\ref{sec-alg-map-hier}) & $1.23\times 10^4 \pm 2.95\times 10^2$ & $6.35\times 10^3 \pm 1.57\times 10^3$ \\
        Shape~(\ref{sec-alg-map-shape}) & $1.58\times 10^4 \pm 2.45\times 10^3$ & $2.04\times 10^4 \pm 9.92\times 10^3$ \\
        BGAR~(\ref{sec-alg-map-bgar}) & $1.24\times 10^4 \pm 3.00\times 10^2$ & $9.65\times 10^3 \pm 2.91\times 10^3$ \\
        \toprule
	\end{tabular}
	\caption{Prediction results on the \texttt{last.fm} dataset. Lower values are better. The mean and standard deviation of each metric are reported over 25 runs.}
	\label{table-last}
\end{table*}

\begin{table*}[t]
	\centering
	\begin{tabular}{ccc}
        \toprule
        Model & KLE-S & KLE-F \\
        \midrule
        GaP (App.~\ref{app-f}) & $8.91 \times 10^4 \pm 2.82 \times 10^3$ & $1.59 \times 10^3 \pm 6.34 \times 10^1$ \\
        Rate~(\ref{sec-alg-map-rate}) & $9.17 \times 10^4 \pm 2.99 \times 10^3$ & $1.62 \times 10^3 \pm 7.57 \times 10^1$ \\
        Hier~(\ref{sec-alg-map-hier}) & $9.11 \times 10^4 \pm 2.81 \times 10^3$ & $1.62 \times 10^3 \pm 9.02 \times 10^1$ \\
        Shape~(\ref{sec-alg-map-shape}) & $9.95 \times 10^4 \pm 3.82 \times 10^3$ & $1.84 \times 10^3 \pm 1.37 \times 10^2$ \\
        BGAR~(\ref{sec-alg-map-bgar}) & $8.99 \times 10^4 \pm 2.80 \times 10^3$ & $1.78 \times 10^3 \pm 8.62 \times 10^1$ \\
        \toprule
\end{tabular}
	\caption{Prediction results on the \texttt{ICEWS} dataset. Lower values are better. The mean and standard deviation of each metric are reported over 25 runs.}
	\label{table-icews}
\end{table*}

All considered temporal models achieve comparable predictive performance, both on smoothing and forecasting tasks, and the slight advantage of one method over the others seems to be data-dependent. The methods ``Rate'' and ``Hier'' rank first or second most of time but not always significantly so. The method ``Shape" tends to achieve worse results than others, but not consistently. This suggests that in the MAP estimation framework, prior distributions do not act as strong regularization terms, and are outweighed by the likelihood term. Moreover, the baseline based on the GaP model also achieves good performance. This might be attributed to the high correlation between successive columns on the datasets, and as such, using adjacent columns for estimation is reasonable.

We conclude this section by saying a few words about computational complexity, which can act as a differentiating criterion, as all models achieve similar predictive performance. The algorithms for chains involving the rate parameters of the Gamma distribution, described in Sections~\ref{sec-alg-map-rate} and~\ref{sec-alg-map-hier} have closed-form update rules for all their variables. This leads to efficient block-descent algorithms. This is in contrast with the algorithms for the model based on the chaining on the shape parameter (Section~\ref{sec-alg-map-shape}), which involves solving $K(N-1)$ equations numerically at each iteration, and the algorithm for BGAR (Section~\ref{sec-alg-map-bgar}), which involves serially solving $2K(N-1)$ order-3 polynomials at each iteration.

\section{Conclusion} \label{sec-ccl}

In this paper, we have reviewed existing temporal NMF models in a unified MAP framework and introduced a new one. These models differ by the choice of the Markov chain structure used on the activation coefficients to induce temporal correlation. We began by studying the previously proposed Gamma Markov chains of the NMF literature, only to find that they all share the same drawback, namely the absence of a well-defined stationary distribution. This leads to problematic behaviors from the generative perspective, because the realizations of the chains are degenerate (although this is not necessarily a problem in MAP estimation). We then introduced a Markovian process from the time series literature, called BGAR(1), which overcomes this limitation, and which, to the best of our knowledge, had never been exploited for learning tasks.

We then derived MAP estimation algorithms in the context of a Poisson likelihood, which allowed for a comprehensive comparison on a prediction task on real datasets. As it turns out, we cannot claim that there is a single model which outperforms all the others. It seems that in our framework, MAP estimation will tend to homogenize the performance of all the models.

Future work will focus on finding a way to perform inference with the BGAR prior for a less restrictive set of hyperparameters, which might increase the performance of this particular model. Moreover, it should be noted that this work can easily be extended to other likelihoods than Poisson thanks to the MM framework. To illustrate this, we present in Appendix~\ref{app-g} the derivation of an algorithm for MAP estimation in a model consisting in an exponential likelihood and BGAR(1) temporal prior. Finally, it would be interesting to carry out similar experimental work within a fully Bayesian estimation paradigm, which might make the differences between the models more striking.

\section*{Acknowledgments}

This work has received funding from the European Research Council (ERC) under the European Union’s Horizon 2020 research and innovation program under grant agreement No 681839 (project FACTORY). Louis Filstroff and Olivier Gouvert were with IRIT, Univ. Toulouse, CNRS, France at the time this research was conducted.

\clearpage

\begin{appendices}

\section{The Beta-Prime distribution} \label{app-a}

Distribution for a continuous random variable in $[0,+\infty[$, with parameters $\alpha > 0$, $\beta > 0$, $p > 0$ and $q > 0$. Its probability density function writes, for $x \geq 0$:
\begin{equation}
f(x;\alpha,\beta,p,q) = \frac{p \left( \frac{x}{q} \right)^{\alpha p -1} \left(1 + \left(\frac{x}{q}\right)^p \right)^{-\alpha - \beta}}{q \text{B}(\alpha, \beta)}.
\end{equation}

\section{BGAR(1) linear correlation} \label{app-b}

We have between two successive values $h_n$ and $h_{n+1}$:
\begin{align}
& \text{corr}(h_n, h_{n+1}) \notag \\
& = \frac{\mathbb{E}(h_n h_{n+1}) - \mathbb{E}(h_n)\mathbb{E}(h_{n+1})}{\sigma(h_n)\sigma(h_{n+1})} \\
& = \frac{\mathbb{E}(h_n(b_{n+1}h_n + \epsilon_{n+1})) - \mathbb{E}(h_n)\mathbb{E}(h_{n+1})}{\sigma(h_n)\sigma(h_{n+1})} \\
& = \frac{\mathbb{E}(b_{n+1})\mathbb{E}(h_n^2) + \mathbb{E}(h_n)\mathbb{E}(\epsilon_{n+1}) - \mathbb{E}(h_n)\mathbb{E}(h_{n+1})}{\sigma(h_n)\sigma(h_{n+1})} \\
& = \frac{\frac{\alpha \rho}{\alpha \rho + \alpha(1-\rho)} \frac{\alpha(\alpha+1)}{\beta^2} + \frac{\alpha}{\beta} \frac{\alpha(1-\rho)}{\beta} - \frac{\alpha}{\beta} \frac{\alpha}{\beta} }{\frac{\alpha}{\beta^2}} \\
& = \rho.
\end{align}

\section{Constrained optimization} \label{app-c}

We want to optimize $G_2(\mathbf{W};\tilde{\mathbf{W}})$ w.r.t. $\mathbf{W}$ s.t. $\sum_f{w_{fk}} = 1$. Rewriting this with Lagrange multipliers $\boldsymbol{\lambda} = [\lambda_1,\dotsc,\lambda_{K}]^{\text{T}}$, this is tantamount to
\begin{equation}
\min_{\mathbf{W}, \boldsymbol{\lambda}} G_2(\mathbf{W};\tilde{\mathbf{W}}) + \sum_k \lambda_k (||\mathbf{w}_k||_1 - 1).
\end{equation}
Deriving w.r.t $w_{fk}$ yields
\begin{equation}
w_{fk} = \frac{p'_{fk}}{q'_k + \lambda_k}. \label{eq-lagr}
\end{equation}
We retrieve the constraint by summing this expression over $f$. This gives the expression of the Lagrange multiplier: ${\lambda_k = \sum_f p'_{fk} - q'_k}$. Substituting this expression into Eq.~\eqref{eq-lagr}, we obtain the following update rule
\begin{equation}
w_{fk} = \frac{p'_{fk}}{\sum_f p'_{fk}}.
\end{equation}

\section{Algorithms with missing values} \label{app-d}

In the context of missing values, let us consider a mask matrix $\mathbf{M}$ of size $F \times N$ such that $m_{fn} = 1$ if the entry $v_{fn}$ is observed and 0 otherwise. The likelihood term can then be written as
\begin{equation}
-\log p(\mathbf{V}|\mathbf{H};\mathbf{W}) = - \sum_{f,n} m_{fn} \log p(v_{fn}|[\mathbf{WH}]_{fn}).
\end{equation}
The auxiliary function $G_1$ of Eq.~\eqref{eq-g1} and $G_2$ of Eq.~\eqref{eq-g2} can then be written is the same way, with
\begin{equation}
p_{kn} = \tilde{h}_{kn} \sum_f w_{fk} \frac{m_{fn} v_{fn}}{[\mathbf{W\tilde{H}}]_{fn}}, \quad q_{kn} = \sum_f m_{fn} w_{fk},
\end{equation}
for $G_1$, and
\begin{equation}
p'_{fk} = \tilde{w}_{fk} \sum_n h_{kn} \frac{m_{fn} v_{fn}}{[\mathbf{\tilde{W}H}]_{fn}}, \quad q'_{kn} = \sum_n m_{fn} h_{kn},
\end{equation}
for $G_2$.

\section{Hyperparameter grids} \label{app-e}

For all methods, we have considered constant hyperparameters w.r.t. $k$ (for example $\alpha_k = \alpha$ for all $k$). Additional details regarding each method can be found in the list below.
\begin{itemize}
    \item For GaP, we have considered a two-dimensional grid for the parameters $\alpha$ and $\beta$. Values were $\alpha = \{0.1, 1, 10 \}$ and $\beta = \{ 0.1, 1, 10 \}$.
	\item For ``Rate'', we have set $\alpha = \beta$, which implies that $\mathbb{E}(h_{kn}|h_{k(n-1)}) = h_{k(n-1)}$. We considered a one-dimensional grid with values $\{ 1.5, 10, 100 \}$.
	\item For ``Hier'', we have set $\alpha_h = \beta_h$, and $\alpha_z = \beta_z$, which implies $\mathbb{E}(z_{kn}|h_{k(n-1)}) = h_{k(n-1)}$ and $\mathbb{E}(h_{kn}|z_{kn}) = z_{kn}$. We considered a two-dimensional grid with values $\alpha_h = \{1.5, 10, 100\}$ and $\alpha_z = \{1.5, 10, 100\}$.
	\item For ``Shape'', we have set $\alpha = \beta$, which implies that $\mathbb{E}(h_{kn}|h_{k(n-1)}) = h_{k(n-1)}$. We considered a one-dimensional grid with values $\{ 0.1, 1, 10 \}$.
	\item For ``BGAR'', we have set $\rho = 0.9$, and considered a two-dimensional grid for the parameters $\alpha$ and $\beta$ (note that setting $\rho = 0.9$ implies $\alpha > 10$ in our MAP framework). Values were $\alpha = \{11, 110, 1100 \}$ and $\beta = \{ 0.1, 1, 10 \}$.
\end{itemize}

\section{MAP estimation in the GaP model} \label{app-f}

The prior distribution on $\mathbf{H}$ is such that
\begin{equation}
    h_{kn} \sim \text{Gamma}(\alpha_k, \beta_k).
\end{equation}
MAP estimation amounts to minimizing
\begin{align}
    C(\mathbf{W}, \mathbf{H}) = & -\log p(\mathbf{V}|\mathbf{H};\mathbf{W}) \\
    & + \sum_{k,n} \left( (1-\alpha_k) h_{kn} + \beta_k h_{kn} \right), \notag
\end{align}
which leads to the following MM update rule \citep{dikmen2012maximum}
\begin{equation}
    h_{kn} = \left\{
    \begin{array}{ll}
    0 & \text{if}~p_{kn} + \alpha_k - 1 \leq 0, \\
    \frac{p_{kn} + \alpha_k - 1}{q_{kn} + \beta_k} & \text{else}.
    \end{array}
    \right.
\end{equation}

\section{BGAR with an exponential likelihood} \label{app-g}

Another popular likelihood used in probabilistic NMF models is the Exponential likelihood  \citep{fevotte2009nonnegative,hoffman2010bayesian} which writes
\begin{equation}
    v_{fn} \sim \text{Exp} \left( \frac{1}{[\mathbf{WH}]_{fn} } \right),
\end{equation}
where $\text{Exp}(\beta)= \text{Gamma}(1,\beta)$ refers to the exponential distribution with mean $1/\beta$. This model underlies so-called Itakura-Saito NMF and has most notably been used in audio signal processing applications. We consider MAP estimation in this model with BGAR(1) prior on $\mathbf{H}$. To do so, we resort to the same MM scheme than what was presented in the beginning of Section~\ref{sec-map}. In this case, the majorization of the likelihood term is known from \citep{cao1999cross,fevotte2011algorithms}. In particular, the function
\begin{equation}
    G(\mathbf{H};\tilde{\mathbf{H}}) = \sum_{k,n} \left( \frac{p_{kn}}{h_{kn}} + q_{kn} h_{kn} \right),
\end{equation}
with the notations
\begin{equation}
    p_{kn} = \tilde{h}_{kn}^2 \sum_f \frac{w_{fk} v_{fn}}{[\mathbf{W} \tilde{\mathbf{H}}]_{fn}^2 }, \quad q_{kn} = \sum_f \frac{w_{fk}}{[\mathbf{W} \tilde{\mathbf{H}}]_{fn}},
\end{equation}
is a tight auxiliary function of $-\log p(\mathbf{V};\mathbf{W},\mathbf{H})$ at $\mathbf{H} = \tilde{\mathbf{H}}$ (up to irrelevant constants).

The exact same constraints on $h_{kn}$ and admissible values of hyperparameters detailed in Section~\ref{sec-alg-map-bgar} apply. The minimization w.r.t. $h_{kn}$ then amounts to solving an order-4 polynomial equation
\begin{equation}
    a_{4,kn} h_{kn}^4 + a_{3,kn} h_{kn}^3 + a_{2,kn} h_{kn}^2 + a_{1,kn} h_{kn}^1 + a_{0,kn} = 0,
\end{equation}
whose coefficients are detailed below.\\

For $h_{k1}$:
\begin{align}
    a_{4,kn} & = 0, \\
    a_{3,kn} & = -(q_{k1} + \beta_k ( 1-b_{k2} )), \\
    a_{2,kn} & = (q_{k1} + \beta_k ( 1-b_{k2} )) d_{k1} - (1-\alpha_k) - (1-\gamma_k), \\
    a_{1,kn} & = (1-\alpha_k) d_{k1} + p_{k1}, \\
    a_{0,kn} & = -p_{k1} d_{k1}.
\end{align}

For $h_{kn}, n \in \{2,\dotsc,N-1 \}$:
\begin{align}
    a_{4,kn} & = -(q_{kn} + \beta_k(1-b_{k(n+1)})), \\
    a_{3,kn} & = (q_{kn} + \beta_k(1-b_{k(n+1)}))(c_{kn} + d_{kn}) - 2(1-\gamma_k), \\
    a_{2,kn} & = -c_{kn} d_{kn} (q_{kn} + \beta_k(1-b_{k(n+1)})) \\
    & \quad + p_{kn} + (1-\gamma_k)(c_{kn} + d_{kn}), \notag \\
    a_{1,kn} & = -p_{kn} (c_{kn} + d_{kn}), \\
    a_{0,kn} & = p_{kn} c_{kn} d_{kn}.
\end{align}

For $h_{kN}$:
\begin{align}
    a_{4,kn} & = 0, \\
    a_{3,kn} & = q_{kN} + \beta_k, \\
    a_{2,kn} & = -(q_{kN} + \beta_k)c_{kN}+ (1-\gamma_k), \\
    a_{1,kn} & = -p_{kN}, \\
    a_{0,kn} & = c_{kN} p_{kN}.
\end{align}

\end{appendices}

\clearpage

\bibliographystyle{apalike}
\bibliography{BibFinal}

\end{document}